%% file: paper.tex
\documentclass{article}
\usepackage{nips10submit_e,times}
\nipsfinalcopy
\input{prel.tex}

\newcommand{\submitted}[1]{}
\newcommand{\final}[1]{#1}

\title{Tight Sample Complexity of Large-Margin Learning}

\author{Sivan Sabato$^1$~~~Nathan Srebro$^2$~~~Naftali Tishby$^1$\\
$^1$ School of Computer Science \& Engineering, The Hebrew University, Jerusalem 91904, Israel\\
$^2$ Toyota Technological Institute at Chicago, Chicago, IL 60637, USA\\
\texttt{\{sivan\_sabato,tishby\}@cs.huji.ac.il, nati@ttic.edu}
}

\begin{document}

\maketitle

\begin{abstract}
  We obtain a tight distribution-specific characterization of the
  sample complexity of large-margin classification with $L_2$
  regularization: We introduce the \mbox{$\gamma$-\kgname}, which is  a
  simple function of the spectrum of a distribution's covariance matrix,
  and show
  distribution-specific upper {\em and} lower bounds on the sample
  complexity, both governed by the $\gamma$-\kgname\ of the source
  distribution.  We conclude that this new quantity tightly
  characterizes the true sample complexity of large-margin
  classification.  The bounds hold for 
  a rich family of sub-Gaussian distributions.
\end{abstract}

\section{Introduction}

In this paper we tackle the problem of obtaining a tight
characterization of the sample complexity which a particular
learning rule requires, in order to learn a particular source distribution.
Specifically, we obtain a tight characterization of the sample
complexity required for large (Euclidean) margin learning to obtain
low error for a distribution $D(X,Y)$, for $X \in \reals^d$, $Y
\in \{\pm 1\}$.


Most learning theory work focuses on upper-bounding the sample
complexity.  That is, on providing a bound
$\overline{m}(D,\epsilon)$ and proving that when using some specific learning rule, if the sample size is
at least $\overline{m}(D,\epsilon)$, an excess error of
at most $\epsilon$ (in expectation or with high probability) can be
ensured. For instance, for large-margin
classification we know that if $P_D[\norm{X}\leq B]=1$, then
$\overline{m}(D,\epsilon)$ can be set to $O(B^2/(\gamma^2\epsilon^2))$
to get true error of no more than $\loss^*_{\gamma} + \epsilon$,
where $\loss^*_{\gamma}=\min_{\norm{w}\leq1}P_D(Y\dotprod{w,X}\leq\gamma)$
is the optimal margin error at margin $\gamma$.

Such upper bounds can be useful for understanding positive aspects of
a learning rule.  But it is difficult to understand deficiencies of a
learning rule, or to compare between different rules, based on upper
bounds alone.  After all, it is possible, and often the case, that the
true sample complexity, i.e.~the actual number of samples required to
get low error, is much lower than the bound. 

Of course, some sample complexity upper bounds are known to be ``tight'' or to have an almost-matching 
lower bound.  This usually
means that the bound is tight as a worst-case upper bound for a
specific class of distributions (e.g.~all those with $P_D[\norm{X}\leq
B]=1$).  That is, there exists {\em some} source distribution for
which the bound is tight.
In other words, the bound concerns some
quantity of the distribution (e.g.~the radius of the support), and is
the lowest possible bound {\em in terms of this quantity}.  But this
is not to say that for any {\em specific} distribution this quantity
tightly characterizes the sample complexity.  For instance, we know that the
sample complexity can be much smaller than the radius of the support
of $X$, if the average norm $\sqrt{\E[\norm{X}^2]}$ is small.
However, $\E[\norm{X}^2]$ is also not a precise characterization of the sample complexity,
for instance in low dimensions.

The goal of this paper is to identify a simple quantity determined by the
distribution that {\em does} precisely characterize the sample
complexity.  That is, such that the actual sample complexity for the
learning rule on this {\em specific} distribution is
governed, up to polylogarithmic factors, by this quantity.

In particular, we present the $\gamma$-\kgname\ 
$k_{\gamma}(D)$. This measure refines both the dimension and the average norm of $X$, and it
can be easily calculated from the covariance matrix of $X$. We show
that for a rich family of ``light tailed'' distributions (specifically,
sub-Gaussian distributions with independent uncorrelated directions -- see
Section \ref{sec:subg}), the number of samples required for learning
by minimizing the $\gamma$-margin-violations is both lower-bounded and
upper-bounded by $\tilde{\Theta}(k_{\gamma})$. More precisely, we
show that the sample complexity $m(\epsilon,\gamma,D)$ required for
achieving excess error of no more than $\epsilon$ can be bounded from
above and from below by:
\begin{equation*}
 \Omega(k_\gamma(D)) \leq m(\epsilon,\gamma,D) \leq
 \tilde{O}(\frac{k_{\gamma}(D)}{\epsilon^2}).
\end{equation*}

As can be seen in this bound, we are {\em not} concerned about
tightly characterizing the dependence of the sample complexity on the
desired error \citep[as done e.g. in ][]{SteinwartSc07}, nor with obtaining tight bounds for very small error
levels.  In fact, our results can be interpreted as studying the
sample complexity needed to obtain error well below random, but
bounded away from zero.  This is in contrast to classical statistics
asymptotic that are also typically tight, but are valid only for
very small $\epsilon$.  As was recently shown by Liang and Srebro
\citep{LiangSr10}, the quantities on which the sample
complexity depends on for very small $\epsilon$ (in the classical
statistics asymptotic regime) can be very different from those for
moderate error rates, which are more relevant for machine learning.

Our tight characterization, and in particular the distribution-specific lower bound on the
sample complexity that we establish, can be used to compare large-margin ($L_2$ regularized) learning to other learning rules.  In
Section \ref{sec:examples} we provide two such examples: we use our lower
bound to rigorously establish a sample complexity gap between $L_1$
and $L_2$ regularization previously studied in \cite{Ng04}, and to show
a large gap between discriminative and generative learning on a Gaussian-mixture distribution.

In this paper we focus only on large $L_2$ margin classification.  But
in order to obtain the distribution-specific lower bound, we develop
novel tools that we believe can be useful for obtaining lower
bounds also for other learning rules.

\subsubsection*{Related work} 

Most work on ``sample complexity lower bounds'' is directed at proving
that under some set of assumptions, there exists a source distribution
for which one needs at least a certain number of examples to learn
with required error and confidence
\cite{AntosLu98,EhrenfeuchtHaKeVa88,GentileHe98}.  This type of a
lower bound does not, however, indicate much on the sample complexity
of other distributions under the same set of assumptions.

As for distribution-specific lower bounds, the classical analysis of
Vapnik \citep[Theorem 16.6]{Vapnik95} provides not only sufficient but
also necessary conditions for the learnability of a hypothesis class
with respect to a specific distribution.  The essential condition is
that the $\epsilon$-entropy of the hypothesis class with respect to
the distribution be sub-linear in the limit of an infinite sample
size.  In some sense, this criterion can be seen as providing a
``lower bound'' on learnability for a specific distribution.  However,
we are interested in finite-sample convergence rates, and would like
those to depend on simple properties of the distribution.  The
asymptotic arguments involved in Vapnik's general learnability claim
do not lend themselves easily to such analysis.

Benedek and Itai \citep{BenedekIt91} show that if the distribution is
known to the learner, a specific hypothesis class is learnable if and
only if there is a finite $\epsilon$-cover of this hypothesis class
with respect to the distribution. Ben-David et
al.~\citep{Ben-DavidLuPa08} consider a similar setting, and prove
sample complexity lower bounds for learning with any data
distribution, for some binary hypothesis classes on the real line. In
both of these works, the lower bounds hold for any algorithm, but only
for a worst-case target hypothesis.  Vayatis and Azencott
\citep{VayatisAz99} provide distribution-specific sample complexity
upper bounds for hypothesis classes with a limited VC-dimension, as a
function of how balanced the hypotheses are with respect to the
considered distributions. These bounds are not tight for all distributions, thus this work also does not provide true distribution-specific sample complexity.

%

\section{Problem setting and definitions}\label{sec:definitions}

Let $D$ be a distribution over $\reals^d \times \{\pm 1\}$.
$D_X$ will denote the restriction of $D$ to $\reals^d$.  We are
interested in linear separators, parametrized by unit-norm vectors in
\mbox{$\ball \triangleq \{w\in \reals^d \mid \norm{w}_2 \leq 1\}$}.  For a predictor
$w$ denote its misclassification error with respect to distribution $D$ by $\loss(w,D) \triangleq \P_{(X,Y)\sim D}[Y\dotprod{w,X} \leq 0]$.  For
$\gamma > 0$, denote the $\gamma$-margin loss of $w$ with respect to
$D$ by $\loss_\gamma(w,D) \triangleq \P_{(X,Y)\sim D}[Y\dotprod{w,X} \leq
\gamma]$.  The minimal margin loss with respect to $D$ is denoted by
$\loss^*_\gamma(D) \triangleq \min_{w \in \ball}\loss_\gamma(w,D)$.  For a
sample $S = \{(x_i,y_i)\}_{i=1}^m$ such that $(x_i,y_i) \in \reals^d
\times \{\pm 1\}$, the margin loss with respect to $S$ is denoted by $
\hat{\loss}_\gamma(w,S) \triangleq \frac{1}{m}|\{i \mid y_i \dotprod{x_i,w}
\leq \gamma\}|$ and the misclassification error is $\hat{\loss}(w,S) \triangleq \frac{1}{m}|\{i \mid y_i \dotprod{x_i,w}
\leq 0\}|$.  In this paper we are concerned with learning by
minimizing the margin loss.  It will be convenient for us to discuss
transductive learning algorithms.  Since many
predictors minimize the margin loss, we define:
\begin{definition}
  A \emph{\textbf{margin-error minimization algorithm}} $\cA$ is an algorithm
  whose input is a margin $\gamma$, a training sample $S =
  \{(x_i,y_i)\}_{i=1}^m$ and an unlabeled test sample $\tilde{S}_X =
  \{\tilde{x}_i\}_{i=1}^m$, which outputs a predictor $\tilde{w} \in
  \argmin_{w\in\ball} \hat{\loss}_\gamma(w,S)$.  We denote the output
  of the algorithm by $\tilde{w} = \cA_\gamma(S,\tilde{S}_X)$.
\end{definition}

We will be concerned with the expected test loss of the algorithm
given a random training sample and a random test sample, each of size $m$,
and define
$
\loss_m(\cA_\gamma,D) \triangleq \E_{S,\tilde{S} \sim
  D^m}[\hat{\loss}(\cA(S,\tilde{S}_X),\tilde{S})] $, where
$S,\tilde{S} \sim D^m$ independently.  For $\gamma >0$, $\epsilon \in
[0,1]$, and a distribution $D$, we denote the \textbf{distribution-specific
sample complexity} by $m(\epsilon,\gamma,D)$: this is 
the minimal sample size such that for any margin-error minimization
algorithm $\cA$, and for any $m \geq m(\epsilon,\gamma,D)$, $
\loss_m(\cA_\gamma,D) - \loss^*_\gamma(D)\leq \epsilon$.

\subsection*{Sub-Gaussian distributions}\label{sec:subg}

We will characterize the distribution-specific sample complexity in
terms of the covariance of $X \sim D_X$.  But in order to do so, we must assume
that $X$ is not too heavy-tailed.  Otherwise, $X$ can have even infinite covariance
but still be learnable, for instance if it
has a tiny probability of having an exponentially large norm.  We will 
thus restrict ourselves to sub-Gaussian distributions. This ensures light tails in all directions, while allowing a sufficiently rich family of distributions, as we presently see. 
We also require a more restrictive condition -- namely that $D_X$ can be rotated to a product distribution over the axes of $\reals^d$. A distribution can always be
rotated so that its coordinates are {\em uncorrelated}.  Here we
further require that they are {\em independent}, as of course holds
for any multivariate Gaussian distribution.

\begin{definition}[See e.g. \cite{Garling07,BuldyginKo98}]
A random variable $X$ is \emph{\textbf{sub-Gaussian with moment $B$}} (or $B$-sub-Gaussian) for $B \geq 0$ if 
\begin{equation}\label{eq:subdef}
\forall t \in \reals, \quad \E[\exp(tX)]\leq \exp(B^2 t^2/2).
\end{equation}
We further say that $X$ is sub-Gaussian with
\emph{\textbf{relative moment}} $\rmom = B/\sqrt{\E[X^2]}$.
\end{definition}

The sub-Gaussian family is quite extensive: For instance, any bounded, Gaussian, or Gaussian-mixture random variable with mean zero is included in this family.

\begin{definition}
  A distribution $D_X$ over $X \in \reals^d$ is \emph{\textbf{independently
    sub-Gaussian}} with relative moment $\rmom$ if there exists some orthonormal basis
  $a_1,\ldots,a_d \in \reals^d$, such that $\dotprod{X,a_i}$ are
  independent sub-Gaussian random variables, each with a relative moment $\rmom$.
\end{definition}

\newcommand{\dfamily}{\cD^\textrm{sg}}

We will focus on the family
$\dfamily_\rmom$ of all independently $\rmom$-sub-Gaussian distributions
in arbitrary dimension, for a small fixed constant $\rmom$. For instance, the family $\dfamily_{3/2}$ includes all
Gaussian distributions, all distributions which are uniform over a
(hyper)box, and all multi-Bernoulli distributions, in addition to other less structured distributions.
Our upper bounds and
lower bounds will be tight up to quantities which depend on $\rmom$, which we will
regard as a constant, but the tightness will not depend on the dimensionality of the space or the variance of the distribution.

\section{The $\gamma$-\kgname}

As mentioned in the introduction, the sample complexity of
margin-error minimization can be upper-bounded in terms of the average
norm $\E[\norm{X}^2]$ by $m(\epsilon,\gamma,D) \leq
O(\E[\norm{X}^2]/(\gamma^2 \epsilon^2))$ \citep{BartlettMe01}.  Alternatively, we can rely only on the dimensionality and conclude $m(\epsilon,\gamma,D) \leq
\tilde{O}(d/\epsilon^2)$ \citep{Vapnik95}.  Thus, although both of these bounds are
tight in the worst-case sense, i.e. they are the best bounds that rely only on the norm or only on the dimensionality respectively, neither is tight in a distribution-specific sense: If the average norm is unbounded while the dimensionality is small, an arbitrarily large gap is created between the true
$m(\epsilon,\gamma,D)$ and the average-norm upper bound. 
The converse happens if the dimensionality is arbitrarily high while the average-norm is bounded.

Seeking a distribution-specific tight analysis, one simple option to
try to tighten these bounds is to consider their minimum,
$\min(d,\E[\norm{X}^2]/\gamma^2)/\epsilon^2$, which, trivially, is also an upper
bound on the sample complexity.  However, this simple combination is also
not tight: Consider a distribution in which there are a few directions with
very high variance, but the combined variance in all other
directions is small.  We will show that in such situations the sample
complexity is characterized not by the minimum of dimension and norm, but by the sum of the number of high-variance dimensions and the average norm in the other directions. This behavior is captured by the \emph{$\gamma$-\kgname}:

\begin{definition}\label{def:limiteddist}
Let $b>0$ and $k$ a positive integer.
\begin{enumerate}
\renewcommand{\theenumi}{(\alph{enumi})}
\item A subset $\cX \subseteq \reals^d$ is \emph{\textbf{$(b,k)$-limited}} if there exists a sub-space $V \subseteq
\reals^d$ of dimension $d-k$ such that $\cX \subseteq \{x\in\reals^{d}
\mid \norm{x'P}^2 \leq b\}$, where $P$ is an orthogonal projection onto $V$.  
\item A distribution $D_X$ over $\reals^d$ is \emph{\textbf{$(b,k)$-limited}} if there exists a sub-space $V \subseteq \reals^d$ of dimension $d-k$ such that 
$
\E_{X\sim D_X}[\norm{X'P}^2] \leq b,
$
with $P$ an orthogonal projection onto $V$.
\end{enumerate}
\end{definition}

\begin{definition}\label{def:kgamma}
The \emph{\textbf{$\gamma$-\kgname}} of a distribution or a set, denoted by
$k_\gamma$, is the minimum $k$ such that the distribution or set is $(\gamma^2k,k)$ limited.
\end{definition}

It is easy to see that $k_{\gamma}(D_X)$ is upper-bounded by $\min(d,\E[\norm{X}^2]/\gamma^2)$. Moreover, it can be much smaller.  For
example, for $X \in \reals^{1001}$ with independent coordinates such that
the variance of the first coordinate is $1000$, but the variance in
each remaining coordinate is $0.001$ we have $k_1=1$ but $d =
\E[\norm{X}^2] = 1001$.  More generally, if $\lambda_1 \geq \lambda_2
\geq \cdots \lambda_d$ are the eigenvalues of the covariance matrix of
$X$, then $k_{\gamma} = \min \{ k \mid \sum_{i=k+1}^d \lambda_i \leq
\gamma^2 k\}$. A quantity similar to $k_\gamma$ was studied previously in \cite{Bousquet02}. $k_\gamma$ is different in nature from some other quantities used for providing sample complexity bounds in terms of eigenvalues, as in \cite{ScholkopfShSmWi99}, 
since it is defined based on the eigenvalues of the distribution and not of the sample. In \secref{sec:lowerbound} we will see that these can be quite different.

In order to relate our upper and lower bounds, it will be useful to relate the $\gamma$-\kgname\ for
different margins.  The relationship is established in the following
Lemma \submitted{(see proof in the supplementary material)}\final{, proved in the appendix}:
\begin{lemma}\label{lem:kgammagrowth}
For $0 < \alpha<1$, $\gamma > 0$ and a distribution $D_X$,
$k_{\gamma}(D_X) \leq k_{\alpha\gamma}(D_X) \leq \frac{2k_\gamma(D_X)}{\alpha^2} + 1.$
\end{lemma}

We proceed to provide a sample complexity upper bound based on the $\gamma$-\kgname.

\section{A sample complexity upper bound using $\gamma$-\kgname}\label{sec:fatshatteringbound}

In order to establish an upper bound on the sample complexity, we will
bound the fat-shattering dimension of the linear functions over a set in terms of the 
$\gamma$-\kgname\ of the set.  Recall that the fat-shattering dimension is a
classic quantity for proving sample complexity upper bounds:

\begin{definition}\label{def:shattered}
  Let $\cF$ be a set of functions $f:\cX \rightarrow \reals$, and let
  $\gamma > 0$.  The set $\{x_1,\ldots,x_m\} \subseteq \cX$ is
  \textbf{$\gamma$-shattered} by $\cF$ if there exist $r_1,\ldots,r_m\in
  \reals$ such that for all $y \in \binm$ there is an $f \in \cF$ such
  that $\forall i\in [m],\:y_i(f(x_i)-r_i) \geq \gamma$.  The \textbf{$\gamma$-fat-shattering dimension} of $\cF$ is the size of the
  largest set in $\cX$ that is $\gamma$-shattered by $\cF$.
\end{definition}
The sample complexity of $\gamma$-loss minimization is bounded by $\tilde{O}(d_{\gamma/8}/\epsilon^2)$ were $d_{\gamma/8}$ is the
$\gamma/8$-fat-shattering dimension of the function class \cite[Theorem
13.4]{AnthonyBa99}. 
Let $\cW(\cX)$ be the class of linear functions restricted to the domain $\cX$. 
For any set we show:

\begin{theorem}\label{thm:fatshattering}
If a set $\cX$ is $(B^2,k)$-limited, then the $\gamma$-fat-shattering dimension of $\cW(\cX)$ is at most $\frac{3}{2}(B^2/\gamma^2 + k+1)$. Consequently, it is also at most $3k_\gamma(\mathcal{X})+1$.
\end{theorem}

\begin{proof}
Let $X$ be a $m\times d$ matrix whose rows are a set of $m$ points in $\reals^d$ which is $\gamma$-shattered. For any $\epsilon > 0$ we can augment $X$ with an additional column to form the matrix $\tilde{X}$ of dimensions $m\times (d+1)$, such that for all $y \in \{-\gamma,+\gamma\}^m$, there is a $w_y \in B^{d+1}_{1+\epsilon}$ such that $\widetilde{X} w_y = y$ \submitted{(the details can be found in the supplementary material)}\final{(the details can be found in the appendix)}. Since $\cX$ is $(B^2,k)$-limited, there is an
orthogonal projection matrix $\tilde{P}$ of size $(d+1)\times (d+1)$ such that $\forall i\in [m],\norm{\tilde{X}'_i P}^2 \leq B^2$ where $\tilde{X}_i$ is the vector in row $i$ of $\tilde{X}$.
Let $\tilde{V}$ be the sub-space of dimension $d-k$ spanned by the columns of $\tilde{P}$.
To bound the size of the shattered set, we show that the projected rows of $\tilde{X}$ on $V$ are `shattered' using projected labels. We then proceed similarly to the proof of the norm-only fat-shattering bound \citep{ChristianiniSh00}.

We have $\tilde{X} = \tilde{X}\tilde{P} + \tilde{X}(I-\tilde{P})$. In addition, $\tilde{X}w_y = y$. 
Thus $y -  \tilde{X}\tilde{P}w_y = \tilde{X}(I-\tilde{P}) w_y$. $I-\tilde{P}$ is a projection onto a $k+1$-dimensional space, thus the rank of $\tilde{X}(I-\tilde{P})$ is at most $k+1$. Let $T$ be an $m\times m$ orthogonal projection matrix onto the subspace orthogonal to the columns of $\tilde{X}(I-\tilde{P})$. This sub-space is of dimension at most $l = m - (k +1)$,
thus $\trace(T) = l$. $T(y -  \tilde{X}\tilde{P}w_y) = T\tilde{X}(I-\tilde{P}) w_y = 0_{(d+1)\times 1}$. Thus $Ty = T\tilde{X}\tilde{P}w_y$ for every $y \in \{-\gamma,+\gamma\}^m$.

Denote row $i$ of $T$ by $t_i$ and row $i$ of $T \tilde{X}\tilde{P}$ by $z_i$.
We have $\forall i\leq m,\:\dotprod{z_i,w_y^1} = t_i y = \sum_{j\leq m} t_i[j] y[j]$. Therefore
$\dotprod{\sum_i z_i y[i], w_y^1} = \sum_{i\leq m}\sum_{j\leq (l+k)} t_i[j] y[i] y[j]$. 
Since $\norm{w_y^1}\leq 1+\epsilon$, $\forall x\in \reals^{d+1}, (1+\epsilon)\norm{x}\geq \norm{x}\norm{w_y^1} \geq \dotprod{x,w_y^1}$. Thus
$
\forall y\in \{-\gamma,+\gamma\}^m,\:(1+\epsilon)\norm{\sum_i z_i y[i]} \geq \sum_{i\leq m}\sum_{j\leq m} t_i[j] y[i] y[j].
$
Taking the expectation of $y$ chosen uniformly at random, we have
\[
(1+\epsilon)\E[\norm{\sum_i z_i y[i]}] \geq \sum_{i,j} \E[t_i[j] y[i] y[j]] = \gamma^2 \sum_{i} t_i[i] = \gamma^2 \trace(T) = \gamma^2 l.
\]
In addition, $\frac{1}{\gamma^2}\E[\norm{\sum_i z_i y[i]}^2] = \sum_{i=1}^l \norm{z_i}^2 = \trace(\tilde{P}'\tilde{X}'T^2\tilde{X}\tilde{P}) \leq \trace(\tilde{P}'\tilde{X}'\tilde{X}\tilde{P}) \leq B^2m$.
From the inequality $E[X^2] \leq \E[X]^2$, it follows that $l^2 \leq (1+\epsilon)^2\frac{B^2}{\gamma^2}m$. Since this holds for any $\epsilon > 0$, we can set $\epsilon = 0$ and solve for $m$. Thus
$m \leq (k+1) + \frac{B^2}{2\gamma^2} + \sqrt{\frac{B^4}{4\gamma^4}+\frac{B^2}{\gamma^2}(k+1)} \leq (k+1) + \frac{B^2}{\gamma^2}+\sqrt{\frac{B^2}{\gamma^2}(k+1)} \leq \frac{3}{2}(\frac{B^2}{\gamma^2} + k+1)$.
\end{proof}

\begin{cor}\label{cor:upperbound}
  Let $D$ be a distribution over $\cX \times \{\pm1\}$, $\cX \subseteq
  \reals^d$. Then
\[
m(\epsilon,\gamma,D) \leq \widetilde O\left(\frac{k_{\gamma/8}(\cX)}{\epsilon^2}\right).
\]
\end{cor}

The corollary above holds only for distributions with bounded support.
However, since sub-Gaussian variables have an exponentially decaying
tail, we can use this corollary to provide a bound for independently
sub-Gaussian distributions as well \submitted{(see supplementary
  material for proof)}\final{(see appendix for proof)}:

\begin{theorem}[Upper Bound for Distributions in $\dfamily_\rmom$]\label{thm:upperboundsg}
  For any distribution $D$ over $\reals^d \times \{\pm1\}$ such that $D_X \in \dfamily_\rmom$,
\[
m(\epsilon,\gamma,D) = \tilde{O}(\frac{\rmom^2 k_{\gamma}(D_X)}{\epsilon^2}).
\]
\end{theorem}

This new upper bound is tighter than norm-only and dimension-only
upper bounds. But does the $\gamma$-\kgname\ characterize the true
sample complexity of the distribution, or is it just another upper
bound? To answer this question, we need to be able to derive sample
complexity lower bounds as well.  We consider this problem in
following section.

\section{Sample complexity lower bounds using Gram-matrix eigenvalues}\label{sec:shattering}

We wish to find a distribution-specific lower bound that depends on the
$\gamma$-\kgname, and matches our upper bound as closely as possible.
To do that, we will link the ability to learn with a margin,
with properties of the data distribution.  The ability to learn is
closely related to the probability of a sample to be shattered, as evident
from Vapnik's formulations of learnability as a function of the $\epsilon$-entropy.  
In the preceding section we used the fact that non-shattering (as
captured by the fat-shattering dimension) implies learnability.  For the lower bound
we use the converse fact, presented below in Theorem \ref{thm:shatterednotlearned}: If a sample can be fat-shattered with a reasonably high probability,
then learning is impossible.  We then relate the
fat-shattering of a sample to the minimal eigenvalue of its Gram matrix.
This allows us to present a lower-bound on the sample complexity using a lower bound on the smallest eigenvalue of the Gram-matrix
of a sample drawn from the data distribution.  
We use the term `$\gamma$-shattered at the origin' to indicate that a
set is $\gamma$-shattered by setting the bias $r \in \reals^m$ (see
\defref{def:shattered}) to the zero vector.  

\begin{theorem}\label{thm:shatterednotlearned}
Let $D$ be a distribution over $\reals^d\times \{\pm 1\}$.
If the probability of a sample of size $m$ drawn from $D_X^m$ to be $\gamma$-shattered at the origin
is at least $\eta$, then there is a margin-error minimization algorithm $\cA$, such that $\loss_{m/2}(\cA_\gamma,D) \geq \eta/2$.
\end{theorem}

\begin{proof}
For a given distribution $D$, let $\cA$ be an algorithm which, for every two input samples $S$ and $\tilde{S}_X$, labels $\tilde{S}_X$ using the separator $w \in \argmin _{w\in \ball}\hat{\loss}_\gamma(w,S)$ that maximizes $\E_{\tilde{S}_Y \in D_Y^m}[\hat{\loss}_\gamma(w,\tilde{S})]$.
For every $x \in \reals^d$ there is a label $y \in \{\pm1\}$ such that $\P_{(X,Y)\sim D}[Y \neq y \mid X = x] \geq \half$. If the set of examples in $S_X$ and $\tilde{S}_X$ together is $\gamma$-shattered at the origin, then $\cA$ chooses a separator with zero margin loss on $S$, but loss of at least $\half$ on $\tilde{S}$. Therefore $\loss_{m/2}(\cA_\gamma,D) \geq \eta/2$.
\end{proof}

The notion of shattering involves checking the existence of a
unit-norm separator $w$ for each label-vector $y \in \{ \pm 1 \}^m$.
In general, there is no closed form for the minimum-norm separator.
However, the following Theorem
provides an equivalent and simple characterization for fat-shattering:

\begin{theorem}\label{thm:shattercond}
  Let $S = (X_1,\ldots,X_m)$ be a sample in $\reals^d$, denote $X$ the
  $m \times d$ matrix whose rows are the elements of $S$.  Then $S$ is
  $1$-shattered iff $X$ is invertible and $\forall y \in \binm,\quad
  y' (XX')^{-1} y  \leq 1$.
\end{theorem}

\submitted{The proof of this theorem is in the supplementary material.}\final{The proof of this theorem is in the appendix.} The main issue in the proof is showing that if a set is shattered, it is also shattered with exact margins, since the set of exact margins $\{ \pm 1 \}^m$ lies in the convex hull of any set of non-exact margins that correspond to all the possible labelings.
We can now use the minimum eigenvalue of the Gram matrix to obtain a
sufficient condition for fat-shattering, after which we present the theorem linking 
eigenvalues and learnability. For a matrix $X$, $\lambda_n(X)$ denotes the $n$'th largest eigenvalue of $X$.
\begin{lemma}\label{lem:lambdam}
Let $S = (X_1,\ldots,X_m)$ be a sample in $\reals^d$, with $X$ as
above.  If $\lambda_m(XX') \geq m$ then $S$ is $1$-shattered at the origin.
\end{lemma}
\begin{proof}
  If $\lambda_m(XX') \geq m$ then $XX'$ is invertible and
  $\lambda_1((XX')^{-1})\leq 1/m$.  For any $y \in \binm$ we have
  $\norm{y}=\sqrt{m}$ and $y' (XX')^{-1} y \leq \norm{y}^2
  \lambda_1((XX')^{-1}) \leq m(1/m) = 1$.  By \thmref{thm:shattercond} the sample is $1$-shattered at the origin.
\end{proof}

\begin{theorem}\label{thm:inductive}
  Let $D$ be a distribution over $\reals^d\times \{\pm 1\}$, $S$ be an
  i.i.d.~sample of size $m$ drawn from $D$, and denote $X_S$ the
  $m\times d$ matrix whose rows are the points from $S$.  If $
  \P[\lambda_m(X_S X_S') \geq m \gamma^2] \geq
  \eta, $ then there exists a margin-error minimization algorithm
  $\cA$ such that $\loss_{m/2}(\cA_\gamma,D) \geq \eta/2$.
\end{theorem}

\thmref{thm:inductive} follows by scaling $X_S$ by $\gamma$, applying Lemma \ref{lem:lambdam} to establish $\gamma$-fat shattering with probability at least $\eta$, then applying
Theorem \ref{thm:shatterednotlearned}.
Lemma \ref{lem:lambdam} generalizes the
requirement for linear independence when shattering using hyperplanes
with no margin (i.e.~no regularization).  For unregularized
(homogeneous) linear separation, a sample is shattered iff it is
linearly independent, i.e.~if $\lambda_m>0$.  Requiring $\lambda_m>m
\gamma^2$ is enough for $\gamma$-fat-shattering.  Theorem
\ref{thm:inductive} then generalizes the simple observation, that
if samples of size $m$ are linearly independent with high probability, there is no
hope of generalizing from $m/2$ points to the other $m/2$ using
unregularized linear predictors.
\thmref{thm:inductive} can thus be used to derive a distribution-specific lower bound. Define:
\[
\underline{m}_{\gamma}(D) \triangleq \half \min \left\{ m \middle|\: \P_{S \sim
    D^m}[\lambda_m(X_S X'_S) \geq m \gamma^2] < \half \right\}
\]
Then for any $\epsilon<1/4-\ell^*_{\gamma}(D)$, we can conclude that 
$
m(\epsilon,\gamma,D) \geq \underline{m}_{\gamma}(D),
$
that is, we cannot learn within reasonable error with less than
$\underline{m}_\gamma$ examples.  Recall that our upper-bound on
the sample complexity from \secref{sec:fatshatteringbound} was
$\tilde{O}(k_{\gamma})$.  The remaining question is whether we can
relate $\underline{m}_\gamma$ and $k_{\gamma}$, to establish
that the our lower bound and upper bound tightly specify the sample complexity.

\section{A lower bound for independently sub-Gaussian distributions}\label{sec:lowerbound}

As discussed in the previous section, to obtain sample complexity
lower bound we require a bound on the value of the smallest eigenvalue
of a random Gram-matrix.  The distribution of this eigenvalue has been
investigated under various assumptions. The cleanest results are in
the case where $m,d \rightarrow \infty$ and $\frac{m}{d}\rightarrow
\beta < 1$, and the coordinates of each example are identically
distributed:
\begin{theorem}[Theorem 5.11 in \cite{BaiSi10}]\label{thm:asym}
Let $X_i$ be a series of  $m_i \times d_i$ matrices whose entries are i.i.d. random variables with mean zero, variance $\sigma^2$ and finite fourth moments. If $\lim_{i\rightarrow \infty}\frac{m_i}{d_i} = \beta < 1$, then
$\lim_{i\rightarrow \infty} \lambda_m(\frac{1}{d}X_iX_i') = \sigma^2(1-\sqrt{\beta})^2.$
\end{theorem}

This asymptotic limit can be used to calculate
$\underline{m}_{\gamma}$ and thus provide a lower bound on the sample
complexity: Let the coordinates of $X\in\reals^d$ be i.i.d.~with
variance $\sigma^2$ and consider a sample of size $m$.  If $d,m$ are
large enough, we have by \thmref{thm:asym}:
\[
\lambda_m(XX') \approx d \sigma^2 (1-\sqrt{m/d})^2  = \sigma^2(\sqrt{d}-\sqrt{m})^2
\]
Solving $\sigma^2(\sqrt{d}-\sqrt{
2\underline{m}_\gamma})^2=2\underline{m}_\gamma\gamma^2$ we get
$\underline{m}_\gamma \approx \half d/(1+\gamma/\sigma)^2$.  We can also
calculate the $\gamma$-\kgname\ for this distribution to get $k_\gamma
\approx d/(1+\gamma^2/\sigma^2)$, and conclude that $\tfrac{1}{4}
k_\gamma \leq
\underline{m}_\gamma \leq \half k_\gamma$.  In this case, then, we are
indeed able to relate the sample complexity lower bound with
$k_\gamma$, the same quantity that controls our upper bound.  This
conclusion is easy to derive from known results, however it holds only
asymptotically, and only for a highly limited set of distributions.
Moreover, since \thmref{thm:asym} holds asymptotically for each
distribution separately, we cannot deduce from it any finite-sample
lower bounds for families of distributions.

For our analysis we require \emph{finite-sample} bounds for the
smallest eigenvalue of a random Gram-matrix.  
Rudelson and Vershynin \citep{RudelsonVe09,RudelsonVe08} provide such
finite-sample lower bounds for distributions with identically
distributed sub-Gaussian coordinates.  In the following Theorem we
generalize results of Rudelson and Vershynin to encompass also
non-identically distributed coordinates. \submitted{The proof of
  \thmref{thm:smallesteigwhpsg} can be found in the supplementary
  material.}\final{The proof of
  \thmref{thm:smallesteigwhpsg} can be found in the appendix.} Based on this theorem we conclude
with \thmref{thm:lowerboundsg}, stated below, which constitutes our final
sample complexity lower bound.

\begin{theorem}\label{thm:smallesteigwhpsg}
Let $\rmom > 0$. There is a constant $\beta > 0$ which depends only on $B$, such
that for any $\delta \in (0,1)$ there exists a number $L_0$, such that for any
independently sub-Gaussian distribution with covariance matrix
$\Sigma \leq I$ and $\trace(\Sigma) \geq L_0$, if each of its independent sub-Gaussian coordinates has relative moment $\rmom$, then for any $m \leq \beta \cdot \trace(\Sigma)$
\[
\P[\lambda_m(X_mX_m') \geq m] \geq 1-\delta,
\]
Where $X_m$ is an $m \times d$ matrix whose rows are independent draws from $D_X$.
\end{theorem}

\begin{theorem}[Lower bound for distributions in $\dfamily_\rmom$]
\label{thm:lowerboundsg}
For any $\rmom >0$, there are a constant $\beta > 0$ and an integer $L_0$ such that for any $D$ such that $D_X \in \dfamily_\rmom$ and $k_\gamma(D_X) > L_0$, for any margin $\gamma > 0$ and any $\epsilon < \frac{1}{4} - \loss^*_\gamma(D)$,
\[
m(\epsilon,\gamma,D) \geq \beta k_\gamma(D_X).
\]
\end{theorem}
\begin{proof}
The covariance matrix of $D_X$ is  clearly diagonal. We assume w.l.o.g. that $\Sigma = \diag(\lambda_1,\ldots,\lambda_d)$ where $\lambda_1 \geq \ldots \geq \lambda_d > 0$.
Let $S$ be an i.i.d. sample of size $m$ drawn from $D$. Let $X$ be the $m\times d$ matrix whose rows are the unlabeled examples from $S$. Let $\delta$ be fixed, and set $\beta$ and $L_0$ as defined in \thmref{thm:smallesteigwhpsg} for $\delta$.
Assume $m \leq \beta (k_\gamma-1)$.

We would like to use \thmref{thm:smallesteigwhpsg} to bound the smallest eigenvalue of $XX'$ with high probability, so that we can then apply \thmref{thm:inductive} to get the desired lower bound. However, \thmref{thm:smallesteigwhpsg} holds only if all the coordinate variances are bounded by $1$. Thus we divide the problem to two cases, based on the value of $\lambda_{k_\gamma+1}$, and apply \thmref{thm:smallesteigwhpsg} separately to each case.

\textbf{Case I:} Assume $\lambda_{k_\gamma+1} \geq \gamma^2$. Then $\forall i\in[k_\gamma],\lambda_i \geq \gamma^2$.
Let $\Sigma_1 = \diag(1/\lambda_1,\ldots,1/\lambda_{k_\gamma},0,\ldots,0)$.	
The random matrix $X\sqrt{\Sigma_1}$ is drawn from an independently sub-Gaussian
distribution, such that each of its coordinates has sub-Gaussian relative moment $\rmom$ and covariance matrix $\Sigma\cdot\Sigma_1 \leq I_d$. In addition, $\trace(\Sigma\cdot\Sigma_1) = k_\gamma \geq L_0$. Therefore \thmref{thm:smallesteigwhpsg} holds for $X\sqrt{\Sigma_1}$, and $\P[\lambda_m(X\Sigma_1 X')\geq m] \geq 1-\delta$.
Clearly, for any $X$, $\lambda_m(\frac{1}{\gamma^2}X X') \geq \lambda_m(X\Sigma_1 X')$. Thus $\P[\lambda_m(\frac{1}{\gamma^2}X X')\geq m] \geq 1-\delta$.\\
\textbf{Case II:} Assume $\lambda_{k_\gamma+1} < \gamma^2$. Then $\lambda_i < \gamma^2$ for all $i \in \{k_\gamma+1,\ldots,d\}$. Let $\Sigma_2 = \diag(0,\ldots,0,1/\gamma^2,\ldots,1/\gamma^2)$,
with $k_\gamma$ zeros on the diagonal.
Then the random matrix $X\sqrt{\Sigma_2}$ is drawn from an independently sub-Gaussian 
distribution with covariance matrix $\Sigma\cdot\Sigma_2 \leq I_d$, such that all its coordinates have sub-Gaussian relative moment $\rmom$.
In addition, from the properties of $k_\gamma$ (see discussion in \secref{sec:definitions}), 
$\trace(\Sigma\cdot\Sigma_2) = \frac{1}{\gamma^2}\sum_{i=k_\gamma+1}^d \lambda_i \geq k_\gamma-1 \geq L_0-1$.
Thus \thmref{thm:smallesteigwhpsg} holds for $X\sqrt{\Sigma_2}$, and so
 $\P[\lambda_m(\frac{1}{\gamma^2}XX')\geq m] \geq \P[\lambda_m(X\Sigma_2X')\geq m] \geq 1-\delta$.

In both cases $\P[\lambda_m(\frac{1}{\gamma^2}XX')\geq m] \geq 1-\delta$ for any $m \leq \beta (k_\gamma-1)$.
By \thmref{thm:inductive}, there exists an algorithm $\cA$ such that for any $m \leq \beta (k_\gamma - 1) - 1$, $\loss_m(\cA_\gamma,D) \geq \half-\delta/2$. Therefore, for any $\epsilon < \half - \delta/2 - \loss^*_\gamma(D)$, we have
$m(\epsilon,\gamma,D) \geq \beta (k_\gamma-1)$. We get the theorem by setting $\delta = \frac{1}{4}$.
\end{proof}

\section{Summary and consequences}\label{sec:examples}

\thmref{thm:upperboundsg} and \thmref{thm:lowerboundsg} provide an upper bound and a lower bound for the sample complexity of any distribution $D$ whose data distribution is in $\dfamily_\rmom$ for some fixed $\rmom > 0$. We can thus draw the following bound, which holds for any $\gamma > 0$ and $\epsilon \in (0,\frac{1}{4} - \loss^*_\gamma(D))$:
\begin{equation}\label{eq:doublebound}
 \Omega(k_\gamma(D_X)) \leq m(\epsilon,\gamma,D) \leq \tilde{O}(\frac{k_{\gamma}(D_X)}{\epsilon^2}).
 \end{equation}
In both sides of the bound, the hidden constants depend only on the constant $\rmom$. This result shows that the true sample complexity of learning each of these distributions is characterized by the $\gamma$-\kgname. An interesting conclusion can be drawn as to the influence of the conditional distribution of labels $D_{Y|X}$: Since \eqref{eq:doublebound} holds for any $D_{Y|X}$, the effect of the direction of the best separator on the sample complexity is bounded, even for highly non-spherical distributions.
We can use \eqref{eq:doublebound} to easily characterize the sample complexity behavior for interesting distributions, and to compare $L_2$ margin minimization to learning methods.

\textbf{Gaps between $L_1$ and $L_2$ regularization in the presence of
  irrelevant features}.
Ng \cite{Ng04} considers learning a single relevant
feature in the presence of many irrelevant features, and compares
using $L_1$ regularization and $L_2$ regularization.  When $\norm{X}_{\infty} \leq 1$,
upper bounds on learning with $L_1$ regularization guarantee a sample
complexity of $O(\log(d))$ for an $L_1$-based learning rule
\cite{Zhang02}.  In order to compare this with the sample complexity of
$L_2$ regularized learning and establish a gap, one must use a {\em
  lower bound} on the $L_2$ sample complexity.  The argument provided by Ng
actually assumes scale-invariance of the learning rule, and is
therefore valid only for {\em unregularized} linear learning.  However,
using our results we can easily establish a lower
bound of $\Omega(d)$ for many specific distributions with
$\norm{X}_{\infty} \leq 1$ and $Y=X[1]\in \{\pm 1\}$.  For instance, when each coordinate
is an independent Bernoulli variable, the distribution is
sub-Gaussian with $\rmom=1$, and $k_1 = \ceil{d/2}$.

\textbf{Gaps between generative and discriminative learning for a Gaussian mixture}.
Consider two classes, each drawn from a unit-variance spherical
Gaussian in a high dimension $\reals^d$ and with a large distance $2v
>> 1$ between the class means, such that $d >> v^4$. Then $\P_D[X|Y=y]
= \mathcal{N}(y v\cdot e_1,I_d)$, where $e_1$ is a unit vector in
$\reals^d$.  For any $v$ and $d$, we have $D_X \in \dfamily_1$.  For
large values of $v$, we have extremely low margin error at $\gamma =
v/2$, and so we can hope to learn the classes by looking for a
large-margin separator.  Indeed, we can calculate $k_\gamma =
\ceil{d/(1+\frac{v^2}{4})}$, and conclude that the sample complexity
required is $\tilde{\Theta}(d/v^2)$.  Now consider a generative
approach: fitting a spherical Gaussian model for each class. This
amounts to estimating each class center as the empirical average of
the points in the class, and classifying based on the nearest
estimated class center.  It is possible to show that for any constant
$\epsilon>0$, and for large enough $v$ and $d$, $O(d/v^4)$ samples are
enough in order to ensure an error of $\epsilon$.  This establishes a
rather large gap of $\Omega(v^2)$ between the sample complexity of the
discriminative approach and that of the generative one. 

To summarize, we have shown that the true sample complexity of large-margin learning of a rich family of specific
distributions is characterized by the
$\gamma$-\kgname. This result allows true comparison between this learning algorithm
and other algorithms, and has various applications, such as semi-supervised learning and feature
construction. The challenge of characterizing true
sample complexity extends to any distribution and any learning algorithm. We
believe that obtaining answers to these questions is of great importance, both to learning theory and to learning applications.

\subsubsection*{Acknowledgments}
The authors thank Boaz Nadler for many insightful discussions, and Karthik Sridharan for pointing out \cite{Bousquet02} to us. 
Sivan Sabato is supported by the Adams Fellowship Program of the Israel Academy of Sciences and Humanities. This work was supported by the NATO SfP grant 982480.

\bibliographystyle{unsrt}
{\small{
\bibliography{bib}
}}

\newpage

\appendix

\section{Proofs for ``Tight Sample Complexity of Large-Margin Learning'' \mbox{(S. Sabato, N. Srebro and N. Tishby)}}

\subsection{Proof of \lemref{lem:kgammagrowth}}
\begin{proof}
The inequality $k_\gamma \leq k_{\alpha\gamma}$ is trivial from the definition of $k_\gamma$. For the other inequality, note first that we can always let $\E_{X\sim D_X}[XX']$ be diagonal by rotating the axes w.l.o.g.~.
Therefore $k_\gamma = \min \{ k \mid \sum_{i=k+1}^d \lambda_i \leq \gamma^2 k\}$.
Since $k_{\gamma} \leq k_{\alpha\gamma}$, we have $\gamma^2 k_\gamma \geq \sum_{i=k_\gamma+1}^d \lambda_i \geq \sum_{i=k_{\alpha\gamma+1}}^d \lambda_i.$
In addition, by the minimality of $k_{\alpha\gamma}$, $\sum_{k_{\alpha\gamma}}^d \lambda_i > \alpha^2\gamma^2(k_{\alpha\gamma}-1)$. Thus $\sum_{i=k_{\alpha\gamma+1}}^d \lambda_i > \alpha^2\gamma^2(k_{\alpha\gamma}-1)-\lambda_{k_{\alpha\gamma}}$. Combining the inequalities we get 
$\gamma^2 k_\gamma > \alpha^2\gamma^2(k_{\alpha\gamma}-1)-\lambda_{k_{\alpha\gamma}}$.
In addition, if $k_{\gamma} < k_{\alpha\gamma}$ then $\gamma^2 k_\gamma \geq \sum_{i=k_{\alpha\gamma}}^d \lambda_i \geq \lambda_{k_{\alpha\gamma}}$. Thus,
either $k_{\gamma} = k_{\alpha\gamma}$ or $2\gamma^2 k_\gamma > \alpha^2\gamma^2(k_{\alpha\gamma}-1)$.
\end{proof}

\subsection{Details omitted from the proof of  \thmref{thm:fatshattering}}
The proof of \thmref{thm:fatshattering} is complete except for the construction of $\tilde{X}$ and $\tilde{P}$ in the first paragraph, which is disclosed here in full, using the following lemma:
\begin{lemma}\label{lem:shatterexact}
Let $S = (X_1,\ldots,X_m)$ be a sequence of elements in $\reals^d$,
and let $X$ be a $m \times d$ matrix whose rows are the elements of $S$.
If $S$ is $\gamma$-shattered, then for  every $\epsilon >0$ there is a column vector $r \in \reals^d$ such that for every $y \in \{\pm\gamma\}^m$  there is a $w_y \in B^{d+1}_{1+\epsilon}$ such that $\widetilde{X} w_y = y$, where $\widetilde{X} = \begin{pmatrix}X  &r\end{pmatrix}$.
\end{lemma}
\begin{proof}
if $S$ is $\gamma$-shattered then there exists a vector $r \in \reals^d$, such that 
for all  $y \in \binm$ there exists $w_y \in \ball$ such that for all $i \in [m], y_i(\dotprod{X_i, w_y} - r_i) \geq \gamma$. For $\epsilon > 0$ define $\widetilde{w}_y = (w_y, \sqrt{\epsilon}) \in B_{1+\epsilon}$, and $\widetilde{r} = r/\sqrt{\epsilon}$, and let $\widetilde{X} = \begin{pmatrix}X  &\widetilde{r}\end{pmatrix}$. For every $y\in \binm$ there is a vector $t_y \in \reals^m$ such that $\forall i\in[m], \frac{1}{\gamma}t_y[i]y[i] \geq 1$, and $\frac{1}{\gamma}\widetilde{X}\widetilde{w}_y = \frac{1}{\gamma}t_y$. As in the proof of necessity in \thmref{thm:shattercond}, it follows that there exists $\widehat{w}_y \in B_{1+\epsilon}$ such that $\frac{1}{\gamma}\widetilde{X}\widehat{w}_y = y$. Scaling $y$ by $\gamma$, we get the claim of the theorem.
\end{proof}

Now,
Let $X$ be a $m\times d$ matrix whose rows are a set of $m$ points in $\reals^d$ which is $\gamma$-shattered. 
By \lemref{lem:shatterexact}, for any $\epsilon > 0$ there exists matrix $\tilde{X}$ of dimensions $m\times (d+1)$ such that the first $d$ columns of $\tilde{X}$ are the respective columns of $X$, and for all $y \in \{pm\gamma\}^m$, there is a $w_y \in B^{d+1}_{1+\epsilon}$ such that $\widetilde{X} w_y = y$. Since $\cX$ is $(B^2,k)$-limited, there exists an orthogonal projection matrix $P$ of size $d\times d$ and rank $d-k$ such that $\forall i\in [m],\norm{X'_i P}^2 \leq B^2$. Let $\widetilde{P}$ be the embedding of $P$ in a $(d+1)\times (d+1)$ zero matrix, so that $\widetilde{P}$ is of the same rank and projects onto the same subspace.
The rest of the proof follows as in the body of the paper.

%
%

\subsection{Proof of \thmref{thm:upperboundsg}}
\begin{proof}[Proof of \thmref{thm:upperboundsg}]
Let $\Sigma = \diag(\lambda_1,\ldots,\lambda_d)$ be the covariance matrix of $D_X$, where $\forall i\in [d-1], \lambda_i \geq \lambda_{i-1}$. 
Define $\cX_\alpha = \{x \in \reals^d \mid \sum_{i=k_\gamma(D_X)+1}^d x[i]^2 \leq \alpha\}$.

Let $\{x_i\}_{i=1}^m$ be an i.i.d. sample of size $m$ drawn from $D_X$. We will select $\alpha$ such that the probability that the whole sample is contained in $\cX_\alpha$ is large.
$\P[\forall i \in [m], x_i \in \cX_\alpha] = (1-\P[x_i\notin \cX_\alpha])^m$.
Let $X \sim D_X$. Then for all $t > 0$, $\P[X \notin \cX_\alpha] = \P[\sum_{i=k_\gamma+1}^d X[i]^2 \geq \alpha] \leq \E[\exp(t \sum_{i=k_\gamma+1}^d X[i]^2)]\exp(-t\alpha). $

Let $\lambda_{\max} = \lambda_{k_\gamma+1}$.
Define $Y \in \reals^d$ such that $Y[i] = X[i]\sqrt{\frac{\lambda_{\max}}{\lambda_i}}$. 
Then $\sum_{i=k_\gamma+1}^d X[i]^2 = \sum_{i=k_\gamma+1}^d \frac{\lambda_i}{\lambda_{\max}} Y[i]^2$, and by the definition of $k_\gamma$,  $\sum_{i=k_\gamma+1}^d \frac{\lambda_i}{\lambda_{\max}}\leq \frac{k_\gamma}{\lambda_{\max}}$. Thus, by \lemref{lem:partition}
\[
\E[\exp(t\sum_{i=k_\gamma+1}^d X[i]^2)] \leq \max_i
(\E[\exp(3tY[i]^2)])^{\ceil{k_\gamma/\lambda_{\max}}}.
\]
For every $i$, $Y[i]$ is a sub-Gaussian random variable with moment $B=\rmom\sqrt{\lambda_{\max}}$. By \cite{BuldyginKo98}, Lemma 1.1.6, $\E[\exp(3tY[i]^2)] \leq (1-6\rmom^2\lambda_{\max}t)^{-\half}$, for $t \in (0,(6\rmom^2\lambda_{max})^{-1})$. Setting $t = \frac{1}{12\rho^2\lambda_{\max}}$,
\[
\P[X \notin \cX_\alpha] \leq 2^{k_\gamma/\lambda_{\max}}\exp(-\frac{\alpha}{12\rmom^2\lambda_{\max}}).
\]

Thus there is a constant $C$ such that for $\alpha(\gamma) \triangleq C\cdot \rmom^2(k_\gamma(D_X) +
\lambda_{\max{}}\ln \frac{m}{\delta})$, $\P[X \notin \cX_{\alpha(\gamma)}] \leq 1-\frac{\delta}{2m}$.  Clearly, $\lambda_{\max}\leq k_\gamma(D_X)$, and $k_{\gamma}(\cX_{\alpha(\gamma)}) \leq \alpha(\gamma)$. Therefore, from \thmref{thm:fatshattering}, the
$\gamma$-fat-shattering dimension of $\cW(\cX_{\alpha(\gamma)})$ is $O(\rmom^2k_\gamma(D_X)\ln\frac{m}{\delta})$. Define $D_\gamma$ to be the distribution such that $\P_{D_\gamma}[(X,Y)] = \P_{D_X}[(X,Y) \mid X \in \cX_{\alpha(\gamma)}]$.
By standard sample complexity bounds \citep{AnthonyBa99}, for any distribution $D$ over $\reals^d \times \{\pm 1\}$, with probability at
least $1-\frac{\delta}{2}$ over samples, $\loss_m(\cA,D) \leq
\tilde{O}(\sqrt{\frac{F(\gamma/8,D)\ln{\frac{1}{\delta}}}{m}})$, where $F(\gamma,D)$ is
the $\gamma$-fat-shattering dimension of the class of linear functions with domain restricted to the support of $D$ in $\reals^d$. Consider $D_{\gamma/8}$. 
Since the support of $D_{\gamma/8}$ is $\cX_{\alpha(\gamma/8)}$, $F(\gamma/8,D_{\gamma/8}) \leq O(\rmom^2k_{\gamma/8}(D_X)\ln\frac{m}{\delta})$.
With probability $1-\delta$ over samples from $D_X$, the sample is
drawn from $D_{\gamma/8}$. In addition, the probability of the
unlabeled example to be drawn from $\cX_{\alpha(\gamma/8)}$ is larger than
$1-\frac{1}{m}$. Therefore $\loss_m(\cA,D) \leq
\tilde{O}(\sqrt{\frac{\rmom^2k_{\gamma/8}(D_X)\ln{\frac{m}{\delta}}}{m}})$.
Setting $\delta = \epsilon/2$ and bounding the expected error, we get $m(\epsilon,\gamma,D) \leq \tilde{O}(\frac{\rmom^2 k_{\gamma/8}(D_X)}{\epsilon^2})$. \lemref{lem:kgammagrowth} allows replacing $k_{\gamma/8}$ with $O(k_{\gamma})$.
\end{proof}

\begin{lemma}\label{lem:partition}
Let $T_1,\ldots,T_d$ be independent 
 random variables such that all the moments $\E[T_i^n]$ for all $i$ are non-negative.
Let $\lambda_1,\ldots,\lambda_d$ be real coefficients such that $\sum_{i=1}^d \lambda_i = L$, and $\lambda_i \in [0,1]$ for all $i\in[d]$.
Then for all $t \geq 0$ 
\[
\E[\exp(t\sum_{i=1}^d \lambda_i T_i)]\leq \max_{i\in[d]}(\E[\exp(3tT_i)])^{\ceil{L}}.
\]
\end{lemma}
\begin{proof}
Let $T_i$ be independent random variables. Then, by Jensen's inequality,
\[
\E[\exp(t \sum_{i=1}^d \lambda_i T_i)] = \prod_{i=1}^d \E[\exp(t \lambda_i T_i)] \leq \prod_{i=1}^d \E[\exp(t T_i \sum_{j=1}^d \lambda_j)]^{\frac{\lambda_i}{\sum_{j=1}^d \lambda_j}} \leq \max_{i\in [d]} \E[\exp(t T_i \sum_{j=1}^d \lambda_j)].
\]

Now, consider a partition $Z_1,\ldots,Z_k$ of $[d]$, and denote $L_j =\sum_{i\in Z_j} \lambda_i$.  Then
by the inequality above,
\[
\E[\exp(t \sum_{i=1}^d \lambda_i T_i)] = \prod_{j=1}^k \E[\exp(t \sum_{i \in Z_j} \lambda_i T_i)] \leq \prod_{j=1}^k \max_{i \in Z_j}\E[\exp(t T_i L_j)].
\]
Let the partition be such that for all $j\in [k]$, $L_j \leq 1$. There exists such a partition such that $L_j < \half$ for no more than one $j$. Therefore, for this partition $L = \sum_{i=1}^d \lambda_i = \sum_{j\in[k]} L_j \geq \half(k-1)$. Thus $k \leq 2L+1$. 

Now, consider $\E[\exp(t T_i L_j)]$ for some $i$ and $j$.  For any random variable $X$
\[
\E[\exp(tX)] = \sum_{n=0}^\infty \frac{t^n\E[X^n]}{n!}.
\] 
Therefore, $\E[\exp(t T_i L_j)] =
\sum_{n=0}^\infty \frac{t^{n}L_j^{n}\E[T_i^{n}]}{(n)!}$. Since $\E[T_i^n]\geq 0$ for all $n$, and $L_j \leq 1$, it follows that $\E[\exp(t T_i L_j)] \leq \E[\exp(t T_i)]$. Thus 
\[
\E[\exp(t \sum_{i=1}^d \lambda_i T_i)] \leq \prod_{j=1}^k \max_{i \in Z_j}\E[\exp(t T_i)] \leq \max_{i \in [d]} \E[\exp(t \sum_{j=1}^k T_i[j])],
\]
where $T_i[j]$ are independent copies of $T_i$.

It is easy to see that $\E[\exp[\frac{1}{a}\sum_{i=1}^a X_i]]\leq \E[\exp[\frac{1}{b}\sum_{i=1}^b X_i]]$, for $a\geq b$ and $X_1,\ldots,X_a$ i.i.d. random variables.
Since $k \geq \ceil{L}$ it follows that
\[
\E[\exp(t \sum_{i=1}^d \lambda_i T_i)] \leq \max_{i \in [d]} \E[\exp(t \sum_{j=1}^k T_i[j])] \leq  \max_{i \in [d]}\E[\exp(t \frac{k}{\ceil{L}} \sum_{j=1}^\ceil{L} T_i[j])].
\]
Since $k \leq 2L+1$ and all the moments of $T_i[j]$ are non-negative, it follows that
\[
\E[\exp(t \sum_{i=1}^d \lambda_i T_i)] \leq \max_{i \in [d]}\E[\exp(t (2+\frac{1}{\ceil{L}}) \sum_{j=1}^\ceil{L} T_i[j])].
\]

\end{proof}

\subsection{Proof of \thmref{thm:shattercond}}
the following lemma, which allows converting the representation
of the Gram-matrix to a different feature space while keeping the separation properties intact.
For a matrix $M$, $M^+$ denotes its pseudo-inverse. If $(M' M)$ is invertible then $M^+ = (B' B)^{-1} B'$.
\begin{lemma}\label{lem:wtilde}
Let $X$ be an $m\times d$ matrix such that $X X'$ is invertible, and $Y$ such that $XX' = Y Y'$.
Let $r \in \reals^m$ be some real vector.
If there exists a vector $\widetilde{w}$ such that $Y \widetilde{w} = r$, then
there exists a vector $w$ such that $X w = r \text{ and } \|w\| = \| P \widetilde{w} \|$,
where $P = Y' Y'^+ = Y' (Y Y')^{-1} Y$ is the projection matrix onto the sub-space spanned by the rows of $Y$.
\end{lemma}
\begin{proof}
Denote $K = XX' = YY'$.
Set $T = Y' X'^+ = Y' K^{-1} X$. Set $w = T' \widetilde{w}$. We have 
$X w = X T' \widetilde{w} = X X' K^{-1} Y\widetilde{w} = Y \widetilde{w} = r.$ In addition,
$
\| w \| = w' w = \widetilde{w}' T T' \widetilde{w}.
$
By definition of $T$, $T T' = Y' X'^+ X^+ Y = Y' K^+ Y = Y' K^{-1} Y = Y' (Y Y')^{-1} Y = Y' Y'^{+} = P.$
Since $P$ is a projection matrix, we have $P^2 = P$. In addition, $P = P'$. Therefore $T T' = P P'$, and so
$
\| w \| = \widetilde{w}' P P' \widetilde{w} = \| P \widetilde{w} \|.
$
\end{proof}
The next lemma will allow us to prove that if a set is shattered at the origin, it can be separated with the exact margin.

\begin{lemma}\label{lem:conv}
Let $R = \{r_y \in \reals^m \mid y \in \binm\}$ such that for all $y\in \binm$ and for all $i \in [m]$, $r_y[i]y[i] \geq 1$. 
Then $\forall y \in \binm, y \in \conv(R)$. 
\end{lemma}
\begin{proof}
We will prove the claim by induction on the dimension $m$. 

{\bf Induction base}: For $m=1$, we have $R = \{(a),(b)\}$ where $a \leq -1$ and $b \geq 1$. Clearly, $\conv{R} = [a,b]$, and the two one-dimensional vectors $(+1)$ and $(-1)$ are in $[a,b]$. 

{\bf Induction step}: For a vector $t = (t[1],\ldots,t[m]) \in \reals^m$, denote by $\bar{t}$ its projection $(t[1],\ldots,t[m-1])$ on $\reals^{m-1}$. Similarly, for a set of vectors $S \subseteq \reals^m$, let \mbox{$\bar{S} = \{\bar{s} \mid s \in S\} \subseteq \reals^{m-1}$}. Define 
\begin{align*}
Y_+ &= \{y \in \binm \mid y[m] = +1\}\\
Y_- &= \{y \in \binm \mid y[m] = -1\}.
\end{align*}
Let $R_+ = \{r_y \mid y\in Y_+ \}$, and similarly for $R_-$. Then $\bar{R}_+$ and $\bar{R}_-$ satisfy the assumptions for $R$ when $m-1$ is substituted for $m$.

Let $y^* \in \binm$. We wish to prove $y^* \in \conv(R)$.
From the induction hypothesis we have $\bar{y}^* \in \conv(\bar{R}_+)$ and $\bar{y}^* \in \conv(\bar{R}_-)$. Thus 
\[
\bar{y}^* = \sum_{y\in Y_+} \alpha_y \bar{r}_y = \sum_{y\in Y_-} \beta_y \bar{r}_y,
\]
where $\alpha_y,\beta_y \geq 0$, $\sum_{y\in Y_+} \alpha_y = 1$, and $\sum_{y\in Y_-} \beta_y =1$. 
Let $y^*_a = \sum_{y\in Y_+} \alpha_y r_y$ and $y^*_b = \sum_{y\in Y_-} \alpha_y r_y$. 
We have that $\forall y \in Y_+, r_y[m] \geq 1$, and $\forall y \in Y_-, r_y[m] \leq -1$. Therefore, $y^*_a[m] \geq 1$ and $y^*_b[m] \leq -1$. In addition, $\bar{y}^*_a = \bar{y}^*_b = \bar{y}$. Hence there is $\gamma \in [0,1]$ such that $y^* = \gamma y^*_a + (1-\gamma)y^*_b$. Since $y^*_a \in \conv(R_+)$ and $y^*_b \in \conv(R_-)$, we have $y^* \in \conv(R)$.
\end{proof}

\begin{proof}[Proof of \thmref{thm:shattercond}]
Let $XX' = U \Lambda U'$ be the SVD of $XX'$, where $U$ is an orthogonal matrix and $\Lambda$ is a diagonal matrix.
Let $Y = U \Lambda^\half$. We have $XX' = Y Y'$. We show that the conditions are sufficient and necessary for the shattering of $S$.

{\bf Sufficient}: Assume $XX'$ is invertible.  Then $\Lambda$ is invertible, thus $Y$ is invertible. For any $y\in \binm$, Let $\widetilde{w} = Y^{-1} y$. We have $Y \widetilde{w} = y$. In addition, $\|\widetilde{w}\|^2 = y' (YY')^{-1}y = y' (XX')^{-1}y \leq 1$. Therefore, by \lemref{lem:wtilde}, 
there exists a separator $w$ such that $Xw = y$ and $\|w\| = \|P \widetilde{w}\| = \| \widetilde{w} \|$.

{\bf Necessary}: If $XX'$ is not invertible then the vectors in $S$ are linearly dependent, thus by standard VC-theory \citep{AnthonyBa99} $S$ cannot be shattered using linear separators. The first condition is therefore necessary. We assume $S$ is $1$-shattered at the origin and show that the second condition necessarily holds. Let $L = \{r \mid \exists w \in \ball, Xw = r\}$. Since $S$ is shattered, 
For any $y \in \binm$ there exists $r_y \in L$ such that $\forall i\in [m], r_y[i]y[i] \geq 1$. By \lemref{lem:conv}, $\forall y\in \binm, y \in \conv(R)$ where $R = \{r_y \mid y\in \binm\}$.
Since $L$ is convex and $R \subseteq L$, $\conv(R) \subseteq L$. Thus for all $y \in \binm$, $y \in L$, that is there exists $w_y \in \reals^m$ such that $Xw_y = y$ and $\|w_y\|\leq 1$.
From \lemref{lem:wtilde} we thus have $\widetilde{w}_y$ such that $Y \widetilde{w}_y = y$ and $\norm{\widetilde{w}_y} = \norm{P w_y} \leq \norm{w_y} \leq 1$. $Y$ is invertible, hence $\widetilde{w}_y = Y^{-1}y$. Thus $y'(XX')^{-1}y = y'(YY')^{-1}y = \norm{\widetilde{w}_y} \leq 1$.
\end{proof}

\subsection{Proof of Theorem \lowercase{\ref{thm:smallesteigwhpsg}}}\label{app:smallesteigwhpsg}

In the proof of \thmref{thm:smallesteigwhpsg} we use the fact   
$\lambda_m(XX') = \inf_{\norm{x}_2=1}\norm{X'x}^2$ and bound the right-hand side
via an $\epsilon$-net of the unit sphere in $\reals^m$, denoted by $S^{m-1} \triangleq \{ x \in \reals^m \mid \norm{x}_2 = 1\}$. 
An $\epsilon$-net of the unit sphere is a set $C \subseteq S^{m-1}$ such that 
$\forall x \in S^{m-1}, \exists x' \in C, \norm{x-x'}\leq \epsilon$. Denote the minimal size of an $\epsilon$-net for $S^{m-1}$ by $\cN_m(\epsilon)$, and by $\cC_m(\epsilon)$ a minimal $\epsilon$-net of $S^{m-1}$, so that $\cC_m(\epsilon) \subseteq S^{m-1}$ and $|\cC_m(\epsilon)| = \cN_m(\epsilon)$.
The proof of \thmref{thm:smallesteigwhpsg}
requires several lemmas.
First we prove a concentration result for the norm of a matrix defined by sub-Gaussian variables.
Then we bound the probability that the squared norm of a vector is small.
\begin{lemma}\label{lem:boundnormsubg}
Let $Y$ be a $d\times m$ matrix with $m \leq d$, such that $Y_{ij}$ are independent sub-Gaussian variables with moment $B$.
Let $\Sigma$ be a diagonal $d\times d$ PSD matrix such that $\Sigma \leq I$. Then for all $t \geq 0$ and $\epsilon \in (0,1)$,
\[
\P[\norm{\sqrt{\Sigma}Y} \geq t] \leq \cN_m(\epsilon)\exp(\frac{\tr(\Sigma)}{2}-\frac{t^2(1-\epsilon)^2}{4B^2}).
\]
\end{lemma}
\begin{proof}
We have $\norm{\sqrt{\Sigma}Y} \leq \max_{x \in \cC_m(\epsilon)}\norm{\sqrt{\Sigma}Yx}/(1-\epsilon)$, see for instance in \cite{Bennet75}.
Therefore, 
\begin{equation}\label{eq:mat2vec}
\P[\norm{\sqrt{\Sigma}Y} \geq t] \leq \sum_{x \in \cC_m(\epsilon)}\P[\norm{\sqrt{\Sigma}Yx} \geq (1-\epsilon)t].
\end{equation}
Fix $x \in \cC_m(\epsilon)$. Let $V = \sqrt{\Sigma}Yx$, and assume $\Sigma = \diag(\lambda_1,\ldots,\lambda_d)$. For $u \in \reals^d$,
\begin{align*}
&\E[\exp(\dotprod{u, V})] = \E[\exp(\sum_{i\in[d]}u_i\sqrt{\lambda}_i \sum_{j\in[m]}Y_{ij}x_j)]
= \prod_{j,i}\E[\exp(u_i\sqrt{\lambda}_i Y_{ij}x_j)]\\
&\quad\leq \prod_{j,i}\exp(u_i^2\lambda_i B^2 x_j^2/2) = \exp(\frac{B^2}{2}\sum_{i\in[d]}u_i^2\lambda_i \sum_{j\in[m]}x_j^2 )\\
&\quad= \exp(\frac{B^2}{2} \sum_{i\in[d]}u_i^2\lambda_i) = \exp(\dotprod{B^2\Sigma u,u}/2).
\end{align*}
Let $s = 1/(4B^2)$. Since $\Sigma \leq I$, we have $s \leq 1/(4B^2\max_{i\in[d]}\lambda_i)$.
Therefore, by \lemref{lem:sgvecmgf} (see \secref{sec:sgvecmgf}), 
\[
\E[\exp(s\norm{V}^2)] \leq \exp(2sB^2\tr(\Sigma)).
\]
By Chernoff's method, 
$\P[\norm{V}^2 \geq z^2] \leq \E[\exp(s\norm{V}^2)]/\exp(sz^2)$. Thus
\[
\P[\norm{V}^2 \geq z^2] \leq \exp(2sB^2\tr(\Sigma) - sz^2) = \exp(\frac{\tr(\Sigma)}{2}-\frac{z^2}{4B^2}).
\]
Set $z = t(1-\epsilon)$. Then for all $x \in S^{m-1}$
\[
\P[\norm{\sqrt{\Sigma}Yx} \geq t(1-\epsilon)] = \P[\norm{V} \geq t(1-\epsilon)] \leq \exp(\frac{\tr(\Sigma)}{2}-\frac{t^2(1-\epsilon)^2}{4B^2}).
\]
Therefore, by \eqref{eq:mat2vec},
\[
\P[\norm{\sqrt{\Sigma}Y} \geq t] \leq \cN_m(\epsilon)\exp(\frac{\tr(\Sigma)}{2}-\frac{t^2(1-\epsilon)^2}{4B^2}).
\]
\end{proof}

\begin{lemma}\label{lem:boundsigmayxsubg}
Let $Y$ be a $d\times m$ matrix with $m \leq d$, such that $Y_{ij}$ are independent 
centered random variables with variance $1$ and fourth moments at most $B$.
Let $\Sigma$ be a diagonal $d\times d$ PSD matrix such that $\Sigma \leq I$.
There exist $\alpha > 0$ and $\eta \in (0,1)$ that depend only on $B$ such that for any $x \in S^{m-1}$
\[
\P[\norm{\sqrt{\Sigma}Yx}^2\leq \alpha\cdot (\tr(\Sigma)-1)] \leq \eta^{\tr(\Sigma)}.
\]
\end{lemma}

To prove \lemref{lem:boundsigmayxsubg} we require \lemref{lem:sumofbounded} \citep[Lemma 2.2]{RudelsonVe08} and \lemref{lem:sigyxboundbelow}, which extends Lemma 2.6 in the same work.
\begin{lemma}\label{lem:sumofbounded}
Let $T_1,\ldots,T_n$ be independent 
non-negative random variables.
Assume that there are $\theta > 0$ and $\mu \in (0,1)$ such that for any $i$, $\P[T_i \leq \theta] \leq \mu$.
There are $\alpha > 0$ and $\eta \in (0,1)$ that depend only on $\theta$ and $\mu$ such that 
$
\P[\sum_{i=1}^n T_i < \alpha n] \leq \eta^n.
$ 
\end{lemma}

\begin{lemma}\label{lem:sigyxboundbelow}
Let $Y$ be a $d\times m$ matrix with $m \leq d$, such that the columns of $Y$ are i.i.d. random vectors. Assume further that $Y_{ij}$ are centered, and have a variance of $1$ and a fourth moment at most $B$.
Let $\Sigma$ be a diagonal $d \times d$ PSD matrix.  
Then for all $x \in S^{m-1}$,
$\P[\norm{\sqrt{\Sigma} Y x} \leq \sqrt{\tr(\Sigma)/2}] \leq 1-1/(196B).$
\end{lemma}
\begin{proof}
Let $x\in S^{m-1}$, and $T_i = (\sum_{j=1}^m Y_{ij} x_j)^2$. Let $\lambda_1,\ldots,\lambda_d$ be the values on the diagonal of $\Sigma$, and let $T_\Sigma = \norm{\sqrt{\Sigma} Y x}^2 = \sum_{i=1}^d \lambda_i T_i$. 
First, since $\E[Y_{ij}] = 0$ and $\E[Y_{ij}] = 1$ for all $i,j$, we have $\E[T_i] = \sum_{i\in [m]} x^2_j\E[Y_{ij}^2] = \norm{x}^2 = 1$.
Therefore $\E[T_\Sigma] = \tr(\Sigma)$.
Second, since $Y_{i1},\ldots,Y_{im}$ are independent and centered, 
we have \citep[Lemma 6.3]{LedouxTa91}
\[
\E[T_i^2] = \E[(\sum_{j\in[m]} Y_{ij} x_j)^4] \leq 16\E_\sigma[(\sum_{j\in[m]} \sigma_jY_{ij} x_j)^4],
\]
where $\sigma_1,\ldots,\sigma_m$ are independent uniform $\{\pm1\}$ variables.
Now, by Khinchine's inequality \citep{NazarovPo00},
\[
\E_\sigma[(\sum_{j\in[m]} \sigma_jY_{ij} x_j)^4] \leq 3\E[(\sum_{j\in[m]} Y^2_{ij} x^2_j)^2]
= 3\sum_{j,k\in[m]} x^2_j x^2_k\E[Y^2_{ij}]\E[Y^2_{ik}].
\]
Now $\E[Y^2_{ij}]\E[Y^2_{ik}] \leq \sqrt{\E[Y^4_{ij}]\E[Y^4_{ik}]} \leq B$.
Thus $\E[T_i^2] \leq 48B\sum_{j,k\in[m]} x^2_j x^2_k = 48B\norm{x}^4 = 48B.$
Thus, 
\begin{align*}
\E[T_\Sigma^2] &= \E[(\sum_{i=1}^d \lambda_i T_i)^2] =  \sum_{i,j=1}^d \lambda_i \lambda_j \E[ T_i T_j] \\
&\leq \sum_{i,j=1}^d \lambda_i \lambda_j \sqrt{\E[T_i^2] \E[T_j^2]} \leq 48B (\sum_{i=1}^d \lambda_i)^2 = 48B\cdot\tr(\Sigma)^2.
\end{align*}
By the Paley-Zigmund inequality \citep{PaleyZy32}, for $\theta \in [0,1]$
\[
\P[T_\Sigma \geq \theta \E[T_\Sigma]] \geq (1-\theta)^2 \frac{\E[T_\Sigma]^2}{\E[T_\Sigma^2]} \geq \frac{(1-\theta)^2}{48B}.
\]
Therefore, setting $\theta = 1/2$, we get $\P[T_\Sigma \leq \tr(\Sigma)/2] \leq 1 - 1/(196B).$
\end{proof}

\begin{proof}[Proof of \lemref{lem:boundsigmayxsubg}]
Let $\lambda_1,\ldots,\lambda_d \in [0,1]$ be the values on the diagonal of $\Sigma$.
Consider a partition $Z_1,\ldots,Z_k$ of $[d]$, and denote $L_j =\sum_{i\in Z_j} \lambda_i$. 
There exists such a partition such that for all $j\in [k]$, $L_j \leq 1$, and for all $j \in [k-1]$, $L_j > \half$. 
Let $\Sigma[j]$ be the sub-matrix of $\Sigma$ that includes the rows and columns whose indexes are in $Z_j$. Let $Y[j]$ be the sub-matrix of $Y$ that includes the rows in $Z_j$. Denote $T_j = \norm{\sqrt{\Sigma[j]} Y[j] x}^2$.
Then
\[
\norm{\sqrt{\Sigma}Yx}^2 = 
\sum_{j\in[k]} \sum_{i\in Z_j} \lambda_i (\sum_{j=1}^m Y_{ij} x_j)^2 = 
\sum_{j\in[k]} T_j.
\]

We have $\tr(\Sigma) = \sum_{i=1}^d \lambda_i \geq \sum_{j\in[k-1]} L_j \geq \half(k-1)$. In addition, $L_j \leq 1$ for all $j \in [k]$. Thus
$
\tr(\Sigma) \leq k \leq 2\tr(\Sigma)+1. 
$
For all $j \in [k-1]$, $L_j \geq \half$, thus by \lemref{lem:sigyxboundbelow}, $\P[T_j \leq 1/4] \leq 1 - 1/(196B)$.
Therefore, by \lemref{lem:sumofbounded} there are $\alpha > 0$ and $\eta \in (0,1)$ that depend only on $B$ such that 
\begin{align*}
&\P[\norm{\sqrt{\Sigma}Yx}^2 < \alpha \cdot(\tr(\Sigma)-1)] \leq 
\P[\norm{\sqrt{\Sigma}Yx}^2 < \alpha (k-1)] \\
&\quad= \P[\sum_{j\in[k]} T_j < \alpha (k-1)] \leq
 \P[\sum_{j\in[k-1]} T_j < \alpha (k-1)] \leq \eta^{k-1} \leq \eta^{2\tr(\Sigma)}.
\end{align*}
The lemma follows by substituting $\eta$ for $\eta^2$.
\end{proof}

\begin{proof}[Proof of \thmref{thm:smallesteigwhpsg}]
We have
\begin{align}\label{eq:lambdam}
&\sqrt{\lambda_m(XX')} = \inf_{x\in S^{m-1}}\norm{X'x} \geq 
\min_{x \in \cC_m(\epsilon)}\norm{X'x}-\epsilon\norm{X'}.
\end{align}
For brevity, denote $L = \tr(\Sigma)$. Assume $L \geq 2$.
Let $m\leq L\cdot\min(1,(c-K\epsilon)^2)$ where $c, K, \epsilon$ are constants that will be set later such that $c- K\epsilon > 0$. By \eqref{eq:lambdam}
\begin{align}
&\P[\lambda_m(XX') \leq m] \leq
\P[\lambda_m(XX') \leq (c-K\epsilon)^2L] \notag\\
&\quad\leq 
  \P[\min_{x\in\cC_m(\epsilon)}\norm{X'x} -
    \epsilon\norm{X'} \leq (c-K\epsilon)\sqrt{L}] \label{eq:proballprev}
    \\ &\quad\leq
  \P[\norm{X'} \geq K\sqrt{L}] + \P[\min_{x\in\cC_m(\epsilon)}\norm{X'x} \leq c\sqrt{L}].\label{eq:proball}
\end{align}
The last inequality holds since the inequality in line (\ref{eq:proballprev}) implies at least one of the inequalities in line (\ref{eq:proball}). We will now upper-bound each of the terms in line (\ref{eq:proball}).
We assume w.l.o.g. that $\Sigma$ is not singular (since zero rows and columns can be removed from $X$ without changing $\lambda_m(XX')$). Define $Y \triangleq \sqrt{\Sigma^{-1}}X' $.
Note that $Y_{ij}$ are independent sub-Gaussian variables with (absolute) moment $\rmom$.
To bound the first term in line (\ref{eq:proball}), note that by \lemref{lem:boundnormsubg}, for any $K > 0$,
\begin{equation*}
\P[\norm{X'} \geq K\sqrt{L}] = \P[\norm{\sqrt{\Sigma}Y} \geq K\sqrt{L}] \leq \cN_m(\half)\exp(L(\half-\frac{K^2}{16\rmom^2})).
\end{equation*}
By \cite{RudelsonVe09}, Proposition 2.1, for all $\epsilon \in[0,1]$, $\cN_n(\epsilon) \leq 2m(1+\frac{2}{\epsilon})^{m-1}.$ Therefore
\[
\P[\norm{X'} \geq K\sqrt{L}] \leq 2m5^{m-1}\exp(L(\half-\frac{K^2}{16\rmom^2})).
\]
Let $K^2 = 16\rmom^2(\frac{3}{2}+\ln(5)+\ln(2/\delta))$.
Recall that by assumption $m \leq L$, and $L \geq 2$.
Therefore
\begin{align*}
&\P[\norm{X'} \geq K\sqrt{L}]  \leq 2m5^{m-1}\exp(-L(1+\ln(5)+\ln(2/\delta)))\notag\\
&\quad\leq 2L5^{L-1}\exp(-L(1+\ln(5)+\ln(2/\delta))).
\end{align*}
Since $L \geq 2$, we have $2L\exp(-L) \leq 1$. Therefore
\begin{align}
\P[\norm{X'} \geq K\sqrt{L}]\leq 2L\exp(-L-\ln(2/\delta))\leq \exp(-\ln(2/\delta)) = \frac{\delta}{2}.\label{eq:prob1}
\end{align}

To bound the second term in line (\ref{eq:proball}), since $Y_{ij}$ are sub-Gaussian with moment $\rmom$, $\E[Y_{ij}^4] \leq 5\rmom^4$ \citep[Lemma 1.4]{BuldyginKo98}.
Thus, by \lemref{lem:boundsigmayxsubg}, there are $\alpha >0$ and $\eta \in (0,1)$ that depend only on $\rmom$ such that for all $x\in S^{m-1}$, $\P[\norm{\sqrt{\Sigma}Yx}^2\leq \alpha (L-1)] \leq \eta^{L}$.
Set $c = \sqrt{\alpha/2}$. Since $L\geq 2$, we have $c\sqrt{L} \leq \sqrt{\alpha(L-1)}$. Thus
\begin{align*}
&\P[\min_{x\in\cC_m(\epsilon)}\norm{X'x} \leq c\sqrt{L}] \leq
\sum_{x\in\cC_m(\epsilon)}\P[\norm{X'x} \leq c \sqrt{L}] \\
&\quad\leq \sum_{x\in\cC_m(\epsilon)}\P[\norm{\sqrt{\Sigma}Yx} \leq \sqrt{\alpha(L-1)}] 
\leq \cN_m(\epsilon)\eta^{L}.
\end{align*}

Let $\epsilon = c/(2K)$, so that $c - K\epsilon > 0$.
Let $\theta = \min(\half,\frac{\ln(1/\eta)}{2\ln(1+2/\epsilon)})$.
Set $L_\circ$ such that $\forall L \geq L_\circ$, $L \geq \frac{2\ln(2/\delta)+2\ln(L)}{\ln(1/\eta)}$. 
For $L \geq L_\circ$ and $m \leq \theta L \leq L/2$,
\begin{align}
\cN_m(\epsilon)\eta^{L} &\leq 2m(1+2/\epsilon)^{m-1}\eta^{L} \notag\\
&\leq L\exp(L(\theta \ln(1+2/\epsilon)-\ln(1/\eta)))\notag\\
&=  \exp(\ln(L) + L(\theta \ln(1+2/\epsilon)-\ln(1/\eta)/2) - L\ln(1/\eta)/2) \notag\\
&\leq  \exp(L(\theta \ln(1+2/\epsilon)-\ln(1/\eta)/2)+\ln(\delta/2))\label{eq:line1}\\
&\leq \exp(\ln(\delta/2)) = \frac{\delta}{2}.\label{eq:line3}
\end{align}
Line (\ref{eq:line1}) follows from $L \geq L_\circ$, and line (\ref{eq:line3})
follows from $\theta \ln(1+2/\epsilon)-\ln(1/\eta)/2 \leq 0$.
Set $\beta = \min\{(c-K\epsilon)^2,1,\theta\}$.
Combining \eqref{eq:proball}, \eqref{eq:prob1} and \eqref{eq:line3} we have
that if $L \geq \bar{L} \triangleq \max(L_\circ,2)$, then
$
\P[\lambda_m(XX') \leq m] \leq \delta
$
for all $m \leq \beta L$.
Specifically, this holds for all $L \geq 0$ and for all $m \leq \beta (L-\bar{L})$. Letting $C = \beta \bar{L}$ and substituting $\delta$ for $1-\delta$ we 
get the statement of the theorem.
\end{proof}

\subsection{\lemref{lem:sgvecmgf}}\label{sec:sgvecmgf}

\begin{lemma}\label{lem:sgvecmgf}
Let $X \in \reals^d$ be a random vector and let $B$ be a PSD matrix such that for all $u \in \reals^d$,
\begin{align*}
\E[\exp(\dotprod{u, V})] \leq \exp(\dotprod{B u,u}/2).
\end{align*}

Then for all $t \in (0,\frac{1}{4\lambda_1(B)}]$, $\E[\exp(t \norm{X}^2)] \leq \exp(2t \cdot \trace(B)).$
\end{lemma}

\begin{proof}[Proof of \lemref{lem:sgvecmgf}]
It suffices to consider diagonal moment matrices:
If $B$ is not diagonal, let $V \in \reals^{d\times d}$ be an orthogonal matrix such that $V B V'$ is diagonal, and let $Y = VX$. We have $\E[\exp(t \norm{Y}^2)] = \E[\exp(t \norm{X}^2)]$ and $\tr(VBV') = \tr(B)$. In addition, for all $u \in \reals^d$, 
\begin{align*}
&\E[\exp(\dotprod{u,Y})] = \E[\exp(\dotprod{V' u, X})]\leq\\
&\quad\exp(\half \dotprod{B V' u, V' u}) = \exp(\half \dotprod{V B V' u, u}).
\end{align*}
Thus assume w.l.o.g. that $B  = \diag(\lambda_1,\ldots,\lambda_d)$ where $\lambda_1 \geq \ldots \geq \lambda_d \geq 0$.

We have $\exp(t\norm{X}^2) = \prod_{i\in[d]}\exp(tX[i]^2)$. In addition, for any $t > 0$ and $x\in \reals$,
$
2\sqrt{\Pi t}\cdot\exp(t x^2) = \int_{-\infty}^\infty \exp(s x-\frac{s^2}{4 t})ds.
$
Therefore, for any $u\in\reals^d$,
\begin{align*}
(2\sqrt{\Pi t})^d \cdot \E[\exp(t\norm{X}^2)] &= \E\left[\prod_{i\in[d]} \int_{-\infty}^\infty \exp(u[i] X[i]-\frac{u[i]^2}{4 t})du[i]\right]\\
&= \E\left[\int_{-\infty}^{\infty} \ldots
  \int_{-\infty}^{\infty} \prod_{i\in[d]} \exp(u[i] X[i]-\frac{u[i]^2}{4 t})du[i]\right]\\
&= \E\left[\int_{-\infty}^{\infty} \ldots
  \int_{-\infty}^{\infty} \exp(\dotprod{u,X}-\frac{\norm{u}^2}{4 t})\prod_{i\in[d]} du[i]\right]\\
&=  \int_{-\infty}^{\infty} \ldots
  \int_{-\infty}^{\infty} \E[\exp(\dotprod{u, X})]\exp(-\frac{\norm{u}^2}{4 t})\prod_{i\in[d]} du[i]
\end{align*}
By the sub-Gaussianity of $X$, the last expression is bounded by
{\allowdisplaybreaks
\begin{align*}
&\leq\int_{-\infty}^{\infty} \ldots \int_{-\infty}^{\infty} \exp(\half \dotprod{Bu,u}-\frac{\norm{u}^2}{4 t})\prod_{i\in[d]} du[i]\\ 
&= \int_{-\infty}^{\infty} \ldots \int_{-\infty}^{\infty} \prod_{i\in[d]} \exp( \frac{\lambda_iu[i]^2}{2} - \frac{u[i]^2}{4t})  du[i]\\
&= \prod_{i\in[d]} \int_{-\infty}^{\infty} \exp(u[i]^2(\frac{\lambda_i}{2}-\frac{1}{4t}))du[i] = \Pi^{d/2}\big(\prod_{i\in[d]}(\frac{1}{4t}-\frac{\lambda_i}{2})\big)^{-\half}.
\end{align*}
}
The last equality follows from the fact that for any $a > 0$, $\int_{-\infty}^{\infty} \exp(-a \cdot s^2)ds= \sqrt{\Pi/a}$,
and from the assumption $t \leq  \frac{1}{4\lambda_1}$.
We conclude that 
\[
\E[\exp(t \norm{X}^2)] \leq (\prod_{i\in[d]}(1-2\lambda_i t))^{-\half} \leq \exp(2t \cdot \sum_{i=1}^d \lambda_i) = \exp(2t\cdot\tr(B)),
\]
where the second inequality holds since  $\forall x \in[0,1]$, $(1-x/2)^{-1}\leq \exp(x)$. 
\end{proof}

\end{document}

%% file: prel.tex
\usepackage{times,amsmath,amsfonts,amssymb,epsfig,stmaryrd,amsthm}
\usepackage{psfrag}
\usepackage[numbers]{natbib}

\newtheorem{theorem}{Theorem}[section]
\newtheorem{lemma}[theorem]{Lemma}
\newtheorem{cor}[theorem]{Corollary}
\newtheorem{definition}[theorem]{Definition}

\renewcommand{\eqref}[1]{Eq.~(\ref{#1})}

\newcommand{\secref}[1]{Section~\ref{#1}}
\newcommand{\thmref}[1]{Theorem~\ref{#1}}
\newcommand{\lemref}[1]{Lemma~\ref{#1}}
\newcommand{\defref}[1]{Def.~\ref{#1}}

\newcommand{\tr}{\textrm{tr}}

\renewcommand{\P}{\mathbb{P}}
\newcommand{\E}{\mathbb{E}}
\newcommand{\reals}{\mathbb{R}}

\newcommand{\diag}{\textrm{diag}}
\newcommand{\conv}{\textrm{conv}}

\newcommand{\half}{{\frac12}}

\newcommand{\cX}{\mathcal{X}}

\newcommand{\cF}{\mathcal{F}}
\newcommand{\cA}{\mathcal{A}}

\newcommand{\cD}{\mathcal{D}}
\newcommand{\cW}{\mathcal{W}}
\newcommand{\cN}{\mathcal{N}}
\newcommand{\cC}{\mathcal{C}}

\newcommand{\ignore}[1]{}
\DeclareMathOperator*{\argmin}{argmin}

\newcommand{\loss}{\ell}

\newcommand{\binm}{\{\pm 1\}^m}

\newcommand{\norm}[1]{\|#1\|}
\newcommand{\dotprod}[1]{\langle #1 \rangle}

\newcommand{\ceil}[1]{{\lceil #1 \rceil}}

\newcommand{\ball}{\mathbf{B}^d_1}

\newcommand{\trace}{\textrm{trace}}

\newcommand{\kgname}{adapted-dimension}

\newcommand{\rmom}{\rho}